\documentclass{article}
\usepackage{arxiv}
\usepackage{amsmath, amsthm, amsfonts, amssymb}
\usepackage[utf8]{inputenc} 
\usepackage[T1]{fontenc}    
\usepackage{hyperref}       
\usepackage{url}            
\usepackage{booktabs}       
\usepackage{amsfonts}       
\usepackage{nicefrac}       
\usepackage{microtype}      
\usepackage{cleveref}       
\usepackage{lipsum}         
\usepackage{graphicx}
\usepackage[numbers]{natbib}
\usepackage{doi}
\usepackage{bm}
\usepackage{cases}

\usepackage{mathrsfs}
\usepackage[noend]{algpseudocode}
\usepackage{algorithmicx,algorithm}
\usepackage{graphicx}
\usepackage{wrapfig}
\usepackage{fancyhdr}
\usepackage{mathtools}
\usepackage{xcolor,color}
\usepackage{esint}
\usepackage{enumerate}
\usepackage{bbm}
\numberwithin{equation}{section}
\newtheorem{theorem}{Theorem}[section]

\newtheorem{lemma}[theorem]{Lemma}
\usepackage{mathrsfs}
\usepackage{appendix}
\usepackage{subcaption}
\usepackage{bm}
\usepackage{booktabs}
\usepackage{multirow}
\theoremstyle{plain}
\newtheorem*{claim*}{Claim}
\newtheorem{thm}{Theorem}[section]

\newtheorem{dfn}{Definition}[section]
\newtheorem{assumption}{Assumption}[section]

\def\<{\langle}
\def\>{\rangle}

\title{Gradient-enhanced deep neural network approximations}
\author{{Xiaodong Feng} \\
	\texttt{xdfeng@lsec.cc.ac.cn} \\
	\And
	{Li Zeng} \\
	\texttt{zengli@lsec.cc.ac.cn} 
	\And{}\\
		LSEC, Institute of Computational Mathematics and Scientific/Engineering
	Computing,\\
	 AMSS, Chinese Academy of Sciences, Beijing, China.\\
}
\begin{document}
	\maketitle
\begin{abstract}

We propose in this work the gradient-enhanced deep neural networks (DNNs) approach for function approximations and uncertainty quantification. More precisely, the proposed approach adopts both the function evaluations and the associated gradient information to yield enhanced approximation accuracy. In particular, the gradient information is included as a regularization term in the gradient-enhanced DNNs approach, for which we present similar posterior estimates (by the two-layer neural networks) as those in the path-norm regularized DNNs approximations. We also discuss the application of this approach to gradient-enhanced uncertainty quantification, and present several numerical experiments to show that the proposed approach can outperform the traditional DNNs approach in many cases of interests.
\end{abstract}
	
\section{Introduction}
In recent years, deep neural networks (DNNs) have been widely used for dealing with scientific and engineer problems, such as function approximations \cite{siegel2020approximation, schwab2019deep, weinan2022barron}, numerical partial differential equations (PDEs) \cite{sirignano2018dgm, weinan2018deep, raissi2019physics}, image classification \cite{he2016deep, litjens2017survey} and uncertainty quantification \cite{qin2021deep, meng2020composite, yang2021b}, to name a few.
Compared to traditional tools such as polynomials \cite{devore1993constructive}, radial basis functions \cite{majdisova2017radial} and kernel methods \cite{liu2020random}, one of the main advantages of DNNs is its potential approximation capacity for high dimensional problems. Unlike classic tools such as polynomial approximations (for which the relevant theoretical analysis results have been well studied), the associated theoretical analysis for DNNs is still in its infancy. Among others, we mention the seminal works \cite{barron1993universal, weinan2022barron} where the concept of "Barron space" was proposed, and some approximation results for DNNs were presented.

In this work, we shall present the gradient-enhanced DNNs approach. More precisely, our approach adopts both the function evaluations and the associated gradient information. This is similar as in the classic Hermite type interpolation. Our main contributions are summarized as follows:
\begin{itemize}
	\item We present the gradient-enhanced DNNs approach, where the gradient information is included as a regularization term.
	\item For the gradien-enhanced DNNs approach, we present similar posterior estimates (via a two-layer neural network) as those in the path-norm regularized DNNs approximations. More precisely, we show that the posterior generalization error can be bounded by $\mathcal{O}\Big(d\|\bm{\theta}\|_{\mathcal{P}}\sqrt{\frac{\ln(2d)}{n}}+\|\bm{\theta}\|_{\mathcal{P}}\frac{\ln(\|\bm{\theta}\|_{\mathcal{P}})}{\sqrt{n}}\Big),$ where $n$ is the number of training points, $\Vert \bm{\theta}\Vert _{\mathcal{P}}$ is the path norm and $d$ is the dimension.
	
	\item We discuss the application of our approach to gradient-enhanced uncertainty quantification, and present several numerical experiments to show that the gradient-enhanced DNNs approach can outperform the traditional DNNs in many cases of interests.
\end{itemize}
We remark that gradient-enhanced polynomial approximations have been proposed for uncertainty quantification \cite{Ling2018A, 2015Enhancing, 2011ORTHOGONAL, 2013Gradient, 2016On}. Meanwhile, a gradient-enhanced physics-informed neural networks was proposed to improve the accuracy and training efficiency of PINNs in \cite{yu2022gradient}, where the gradient information (i.e. the associated adjoint equation) is included to yield a modified PINNs loss function.

The rest of this paper is organized as follows. In Section 2, we set up the problem and present some preliminaries. In Section 3 we present the error estimations in two-layer neural network for gradient regularized approximation problem. Applications to uncertainty quantification are discussed in Section 4. Finally, we give some concluding remarks in Section 5.

\section{Preliminaries}
\subsection{Problem setup}
We begin the discussion by considering function approximations via DNNs with labeled data. In particular, we consider the target function $f^*:\mathbb{R}^d\to\mathbb{R}.$  And we assume that the following data are available: $\{\bm{x}_i,y_i, \bm{y}'_i\}_{i=1}^n$.
Here $y_i$ and $\bm{y}'_i$ are the functional evaluations and gradient evaluations, respectively. Namely,
\begin{equation}
	\left\{\begin{split}
		y_i &= f^*(\bm{x}_i), \quad i=1,...,n.\\
		\bm{y}'_i &=\nabla f^*(\bm{x}_i), \quad i=1,...,n.
	\end{split}
	\right.
\end{equation}
For simplicity, we assume that the data $\{\bm{x}_i\}_i$ lie in $\mathcal{X}=[-1,1]^d$ and $0\leq f^*\leq1$.  We shall show our analysis results via a two-layer neural network, for which the nonlinear function can be defined as:
\begin{equation}
	f(\bm{x};\bm{\theta})=\sum\limits_{i=1}^ma_k\sigma(\bm{w}_k^T\bm{x}),
	\label{model}
\end{equation}
where $\bm{w}_k\in\mathbb{R}^d, \sigma:\mathbb{R}\to\mathbb{R}$ is a nonlinear activation function, and $\bm{\theta}=\{(a_k, \bm{w}_k)\}^m_{k=1}$ is the unknown parameter. We define a truncated form of $f(\bm{x};\bm{\theta})$ through
$$T f(\bm{x};\bm{\theta}) = \max\big\{\min\{f(\bm{x};\bm{\theta}),1\},0\big\}.$$
 By an abuse of notation, in the following
we still use $f(\bm{x};\bm{\theta})$ to denote $T f(\bm{x};\bm{\theta})$. For the training $\{\bm{x}_i, y_i\}_{i=1}^n$, the population risk can be defined by
\begin{equation}
	L(\bm{\theta})=\mathbb{E}_{\bm{x},y}\big[\ell(f(\bm{x};\bm{\theta}),y)\big].
	\label{risk}
\end{equation}
The discrete empirical risk with the training data yields
\begin{equation}
	L_n(\bm{\theta})=\frac{1}{n}\sum\limits_{i=1}^n \ell(f(\bm{x}_i;\bm{\theta}),y_i),
	\label{emperical risk}
\end{equation}
where $\ell(f(\bm{x}),y)=\frac{1}{2}(f(\bm{x})-y)^2$. For the gradient information, we consider
\begin{equation}
L'(\bm{\theta})=\mathbb{E}_{\bm{x},y}\big[\tilde{\ell}(\nabla f(\bm{x};\bm{\theta}),\bm{y}')\big].
\label{dL}
\end{equation}
where $\tilde{\ell}(\nabla f(\bm{x};\bm{\theta}), \bm{y}')=\big\|\nabla f(\bm{x};\bm{\theta})-\bm{y}'\big\|_2$ and $\big\|\cdot\big\|_q$ indicates the $\ell_q$ norm of a vector.
Similarly, the discrete empirical risk with the training data $\{\bm{x}_i, y_i,\bm{y}'_i\}_{i=1}^n$ yields
\begin{equation}
 {L}'_n(\bm{\theta})=\frac{1}{n}\sum\limits_{i=1}^n\big[\tilde{\ell}(\nabla f(\bm{x}_i;\bm{\theta}), \bm{y}'_i)\big]^2.
	\label{dLn}
\end{equation}

Moreover, the path norm of the two-layer neural network is defined as:
	\begin{equation}
	\|\bm{\theta}\|_{\mathcal{P}}:=\sum\limits_{k=1}^m|a_k|\|\bm{w}_k\|_1.
\end{equation}
We are new ready to present the so called gradient-enhanced DNNs approach.

\begin{dfn}[\bf{Gradient-enhanced DNNs model}]
	
	For a two-layer neural network $f(\cdot;\bm{\theta})$ of width $m$,  the gradient regularized risk is defined as follows:
	$$J_{n,\beta}(\bm{\theta}):=L_n(\bm{\theta})+\beta\cdot L'_n(\bm{\theta}).$$
	The corresponding regularized estimator is defined as
	$$\bm{\theta}_{n,\beta}=\arg\min J_{n,\beta}(\bm{\theta}).$$
	Note that the minimizers is not necessarily unique, and $\bm{\theta}_{n,\beta}$ should be understood as any of the minimizers.
\end{dfn}
The above approach can be viewed as an extension of the classic Hermite interpolation \cite{spitzbart1960generalization, zongmin1992hermite}, and is motivated by applications such as gradient-enhanced uncertainty quantification \cite{Ling2018A, 2015Enhancing, 2011ORTHOGONAL, 2013Gradient, 2016On}.
\subsection{Barron space}
We now provide a brief overview of Barron space \cite{barron}. Let $\mathbb{S}^{d-1}:=\big\{\bm{w}\in \mathbb{R}^{d}\mid \|\bm{w}\|_1=1\big\}.$ Let $\mathcal{F}$ be the Borel $\sigma$-algebra on $\mathbb{S}^{d-1}$ and $\mathbb{P}(\mathbb{S}^{d-1})$ be the collection of the probability measures on $\big(\mathbb{S}^{d-1},\mathcal{F}\big)$. Let $\mathcal{B}(\mathcal{X})$ be the collection of functions that admit the following integral representation:
\begin{equation}
	f(\bm{x})=\int_{\mathbb{S}^{d-1}}a(\bm{w})\sigma(\<\bm{w},\bm{x}\>)\mathrm{d}\bm{\pi}(\bm{w}),\quad \forall \bm{x}\in \mathcal{X},
\end{equation}
where $\bm{\pi}\in \mathbb{P}(\mathbb{S}^{d-1})$, and $a(\cdot)$ is a measurable function with respect to $(\mathbb{S}^d,\mathcal{F})$. For any $f\in\mathcal{B}(\mathcal{X})$ and $p\geq1$, we define the following norm
\begin{equation}
	\gamma_p(f):=\inf\limits_{(a,\bm{\pi})\in\Theta_f}\left(\int_{\mathbb{S}^{d-1}}|a(\bm{w})|^p\mathrm{d}\bm{\pi}(\bm{w})\right)^{1/p},
\end{equation}
where
$$\Theta_f=\left\{(a,\bm{\pi})\,\Big|\, f(\bm{x})=\int_{\mathbb{S}^{d-1}}a(\bm{w})\sigma(\<\bm{w},\bm{x}\>)\mathrm{d}\bm{\pi}(\bm{w})\right\}.$$

\begin{dfn}[\bf{Barron space} \cite{barron}]
	The Barron space is defined as
	$$\mathcal{B}_p(\mathcal{X}):=\Big\{f\in\mathcal{B}(\mathcal{X})\,\big|\, \gamma_p(f)<\infty\Big\}.$$
\end{dfn}

We next present several assumptions.
\begin{assumption}\label{assumption}
		Throughout the paper we assume that
		\begin{itemize}
		\item $\mathcal{X}=[-1,1]^d$ and $0\leq f^*\leq1$.
		\item The derivative of ${f^*}(\bm{x})$ is bounded by a constant $D$.
		\item The activation function $\sigma$ is scaling invariant, namely $\sigma(k\bm{x})=k\sigma(\bm{x})$, and satisfies $|\sigma(\bm{x})|\leq C_1|\bm{x}|, |\sigma'|\leq C_2$ and $\sigma'$ is Lipschitz continuous with a positive constant $C_3$. In particular, we use $ReLU$ as activation function with constant $C_1=C_2=C_3=1$.
		\item $\ln(2d)\geq 1$, here $d$ is the dimension of input data.
	\end{itemize}
\end{assumption}
\section{Error estimates for the gradient-enhanced DNNs approach}

In this section, we present the error estimates for the gradient-enhanced DNNs approach. To this aim, we first present the following Theorem.
\begin{thm}
	\label{appro_thm}
	For any $f\in\mathcal{B}_2(\mathcal{X})$, there exists a two-layer neural network $f(\cdot;\tilde{\bm{\theta}})$ of width $m$, such that
	\begin{align}
		\mathbb{E}_{\bm{x}}\left[\big(f(\bm{x})-f(\bm{x};\tilde{\bm{\theta}})\big)^2\right]\leq\frac{3\gamma_2^2(f)}{m}\label{appro_eq},\\
		\mathbb{E}_{\bm{x}}\left[\left\|\nabla f(\bm{x})-\nabla f(\bm{x};\tilde{\bm{\theta}})\right\|^2_2\right]\leq\frac{7\gamma_2^2(f)}{m} \label{appro_eqd},\\
		\big\|\tilde{\bm{\theta}}\big\|_{\mathcal{P}}\leq2\gamma_2(f).
	\end{align}
\end{thm}

\begin{proof}
	Let $(a,\bm{\pi})$ be the best representation of $f$, i.e., $\gamma^2_2(f)=\mathbb{E}_{\bm{\pi}}\big[|a(\bm{w})|^2\big]$.
	And let $U=\{\bm{w}_j\}^m_{j=1}$ be i.i.d. random variables drawn from $\bm{\pi}(\cdot)$. Define
	\begin{equation*}
	\hat{f}_U(\bm{x})\coloneqq\frac{1}{m}\sum\limits_{j=1}^ma(\bm{w}_j)\sigma(\<\bm{w}_j,\bm{x}\>).
	\end{equation*}
	Then the derivative of $f$ with respect to $\bm{x}$ is
	\begin{equation*}
	\nabla \hat{f}_U(\bm{x})=\frac{1}{m}\sum\limits_{j=1}^ma(\bm{w}_j)\sigma'(\<\bm{w}_j,\bm{x}\>)\bm{w}_j^T.
	\end{equation*}
	Let $L^1_U=\mathbb{E}_{\bm{x}}\Big[\big|\hat{f}_U(\bm{x})-f(\bm{x})\big|^2\Big]$,  $L^2_U=\mathbb{E}_{\bm{x}}\Big[\big\|\nabla \hat{f}_U(\bm{x})-\nabla f(\bm{x})\big\|_2^2\Big]$. Then we have
	\begin{align*}
		\mathbb{E}_U\big[L_U^1\big] & =  \mathbb{E}_{\bm{x}}\mathbb{E}_U\left[\big|\hat{f}_U(\bm{x})-f(\bm{x})\big|^2\right]\\
		& = \mathbb{E}_{\bm{x}}\mathbb{E}_U\bigg[\bigg|\frac{1}{m}\sum\limits_{j=1}^ma(\bm{w}_j)\sigma(\<\bm{w}_j,\bm{x}\>)-f(\bm{x})\bigg|^2\bigg]\\
		& =  \frac{1}{m^2} \mathbb{E}_{\bm{x}}\mathbb{E}_U\bigg[
		\bigg(\sum\limits_{j=1}^m\big(a(\bm{w}_j)\sigma(\<\bm{w}_j,\bm{x}\>)-f(\bm{x})\big)\bigg) \bigg(\sum\limits_{i=1}^m\big(a(\bm{w}_i)\sigma(\<\bm{w}_i,\bm{x}\>)-f(\bm{x})\big)\bigg)\bigg]\\
		& =  \frac{1}{m^2} \mathbb{E}_{\bm{x}}\sum\limits_{i,j=1}^m\mathbb{E}_{\bm{w}_i,\bm{w}_j}\Big[\big(a(\bm{w}_j)\sigma(\<\bm{w}_j,\bm{x}\>)-f(\bm{x})\big) \big(a(\bm{w}_i)\sigma(\<\bm{w}_i,\bm{x}\>)-f(\bm{x})\big)\Big]\\
		& =  \frac{1}{m} \mathbb{E}_{\bm{x}}\mathbb{E}_{\bm{w}}\Big[\big(a(\bm{w})\sigma(\<\bm{w},\bm{x}\>)-f(\bm{x})\big)^2\Big] \\
		& \leq  \frac{1}{m} \mathbb{E}_{\bm{x}}\mathbb{E}_{\bm{w}}\Big[\big(a(\bm{w})\sigma(\<\bm{w},\bm{x}\>)\big)^2\Big]\\
		& \leq  \frac{1}{m} \mathbb{E}_{\bm{x}}\mathbb{E}_{\bm{w}}\Big[|a(\bm{w})|^2|\sigma(\<\bm{w},\bm{x}\>)|^2\Big]\\
		& \leq  \frac{1}{m} \mathbb{E}_{\bm{x}}\mathbb{E}_{\bm{w}}\big[|a(\bm{w})|^2\big]= \frac{1}{m}\mathbb{E}_{\bm{w}}\big[|a(\bm{w})|^2\big].
	\end{align*}
	Since $\|\bm{w}_j\|_1=1$ and $\bm{x}\in[-1,1]^d$, we have $\<\bm{w}_j,\bm{x}\>\leq\<\bm{w}_j,\bm{x}_0\>\leq\|\bm{w}_j\|_1=1$ where $\bm{x}_0=(\bm{x}_{01},\bm{x}_{02},\dots, \bm{x}_{0d})$ and $\bm{x}_{0i}=sgn(\bm{w}_{ji})$, which implies the last inequality. Meanwhile,
	\begin{align*}
		\mathbb{E}_U[L_U^2] & = \mathbb{E}_{\bm{x}}\mathbb{E}_U\left[\big\|\nabla\hat{f}_U(\bm{x})-\nabla f(\bm{x})\big\|_2^2\right]\\
		& = \mathbb{E}_{\bm{x}}{E}_U\bigg[\bigg\|\frac{1}{m}\sum\limits_{j=1}^ma(\bm{w}_j)\sigma'(\<\bm{w}_j,\bm{x}\>)\bm{w}_j^T-\nabla f(\bm{x})\bigg\|_2\bigg]^2\\
		& = \frac{1}{m^2} \mathbb{E}_{\bm{x}}\mathbb{E}_U\bigg[\bigg(\sum\limits_{j=1}^m\big(a(\bm{w}_j)\sigma'(\<\bm{w}_j,\bm{x}\>)\bm{w}_j^T-\nabla f(\bm{x})\big)\bigg)\cdot \bigg(\sum\limits_{i=1}^m\big(a(\bm{w}_i)\sigma'(\<\bm{w}_i,\bm{x}\>)\bm{w}_i^T-\nabla f(\bm{x})\big)\bigg)\bigg]\\
		& = \frac{1}{m^2} \mathbb{E}_{\bm{x}}\sum\limits_{i,j=1}^m\mathbb{E}_{\bm{w}_i,\bm{w}_j}\Big[\big(a(\bm{w}_j)\sigma'(\<\bm{w}_j,\bm{x}\>)\bm{w}_j^T-\nabla f(\bm{x})\big)^T \big(a(\bm{w}_i)\sigma'(\<\bm{w}_i,\bm{x}\>)\bm{w}_i^T-\nabla f(\bm{x})\big)\Big]\\
		& = \frac{1}{m} \mathbb{E}_{\bm{x}}\mathbb{E}_{\bm{w}}\left[\big\|a(\bm{w})\sigma'(\<\bm{w},\bm{x}\>)\bm{w}^T-\nabla  f(\bm{x})\big\|_2^2\right] \\
		& \leq \frac{1}{m} \mathbb{E}_{\bm{x}}\mathbb{E}_{\bm{w}}\left[\big\|a(\bm{w})\sigma'(\<\bm{w},\bm{x}\>)\bm{w}^T\big\|_2^2\right]\\
		& \leq \frac{1}{m} \mathbb{E}_{\bm{x}}\mathbb{E}_{\bm{w}}\left[|a(\bm{w})|^2\|\bm{w}\|_2^2\right]\\
		& \leq \frac{1}{m} \mathbb{E}_{\bm{x}}\mathbb{E}_{\bm{w}}\left[|a(\bm{w})|^2\|\bm{w}\|_1^2\right] = \frac{1}{m} \mathbb{E}_{\bm{x}}\mathbb{E}_{\bm{w}}\left[|a(\bm{w})|^2\right]= \frac{1}{m}\mathbb{E}_{\bm{w}}\left[|a(\bm{w})|^2\right].
	\end{align*}
	Denote the path norm of $\hat{f}_U(\bm{x})$ by $A_U$, we have $\mathbb{E}[A_U]=\gamma_1(f)\leq\gamma_2(f)$.
	
	Define events $E_1=\left\{L^1_U<\frac{3\gamma_2^2(f)}{m}\right\}, E_2=\left\{L^2_U\le\frac{7\gamma_2^2(f)}{m}\right\}, E_3=\left\{A_U<2\gamma_2(f)\right\}$. Using Markov's inequality, we have
	\begin{align*}
		\mathbb{P}(E_1)=1-\mathbb{P}\left(\bigg\{L^1_U\geq\frac{3\gamma^2_2(f)}{m}\bigg\}\right)\geq 1-\frac{\mathbb{E}_U[L_U^1]}{3\gamma^2_2(f)/m}\geq1-\frac{\gamma_2(f)^2/m}{3\gamma^2_2(f)/m}=\frac{2}{3},\\
	\mathbb{P}(E_2)=1-\mathbb{P}\left(\bigg\{L^2_U\geq\frac{7\gamma^2_2(f)}{m}\bigg\}\right)\geq 1-\frac{\mathbb{E}_U[L_U^2]}{7\gamma^2_2(f)/m}\geq1-\frac{\gamma_2(f)^2/m}{7\gamma^2_2(f)/m}=\frac{6}{7},\\
		\mathbb{P}(E_3)=1-\mathbb{P}\left(\{A_U\geq2\gamma_2(f)\}\right)\geq1-\frac{\mathbb{E}_U[A_U]}{2\gamma_2(f)}\geq1-\frac{\gamma_2(f)}{2\gamma_2(f)}=\frac{1}{2}.
	\end{align*}
	
	Therefore, let $E_4=E_1\cap E_3$, then $\mathbb{P}(E_1\cap E_3)\geq \mathbb{P}(E_1)+\mathbb{P}(E_3)-1=\frac{1}{6}$. Hence,
	$$\mathbb{P}(E_1\cap E_2\cap E_3)=\mathbb{P}(E_4\cap E_2)\geq \mathbb{P}(E_4)+\mathbb{P}(E_2)-1\geq \frac{1}{6}+\frac{6}{7}-1=\frac{1}{42}>0.$$
	i.e., $E_1\cap E_2\cap E_3\neq\varnothing$, which completes the proof.
\end{proof}
\subsection{A posteriori error estimation}

In order to give a posteriori error estimation, we need to introduce the definition of Rademacher complexity and review some related conclusions.
\begin{dfn}[\bf{Rademacher complexity} \cite{shalev}]
	Let $S=\big\{\bm{x}_1,\bm{x}_2,\cdots,\bm{x}_n\big\}$ be $n$ i.i.d samples, $\mathcal{F}\circ S$ be the set of all possible evaluations a function $f\in\mathcal{F}$ can achieve on a sample $S$, namely,
	$$\mathcal{F}\circ S=\Big\{\big(f(\bm{x}_1),f(\bm{x}_2),\dots,f(\bm{x}_n)\big)\,\big|\,f\in\mathcal{F}\Big\}.$$
	Let each component of random variable $\xi$ be i.i.d. according to $\mathbb{P}[\xi_i=1]=\mathbb{P}[\xi_i=-1]=\frac{1}{2}$. Then, the Rademacher complexity of $\mathcal{F}$ with respect to $S$ is defined as follows:
	\begin{equation}
		\mathcal{R}_n(\mathcal{F}\circ S):=\frac{1}{n}\mathbb{E}_{\bm{\xi}\sim\{\pm1\}^n}\bigg[\sup\limits_{f\in\mathcal{F}}\sum\limits_{i=1}^n\xi_if(\bm{x}_i)\bigg].
	\end{equation}
\end{dfn}

\begin{lemma}[\bf{Lemma 26.11 of \cite{shalev}}]
	\label{vector_rn}
	Let $S=\big\{\bm{x}_1,\bm{x}_2,\dots,\bm{x}_n\big\}$ be vectors in $\mathbb{R}^d$, $\mathcal{H}_1=\big\{\bm{x}\mapsto\<\bm{w},\bm{x}\>\,\big|\, \|\bm{w}\|_1\leq1\big\}$. Then,
	$$ \mathcal{R}_n(\mathcal{H}_1\circ S)\leq\max\limits_i\|\bm{x}_i\|_\infty\sqrt{\frac{2\ln(2d)}{n}}.$$
\end{lemma}

\begin{lemma}[\bf{Lemma 26.9 of \cite{shalev}}]
	\label{lipschitz}
	For each $i\in[n]=\big\{1,2,\cdots,n\big\}$, let $\phi_i:\mathbb{R}\mapsto\mathbb{R}$ be a $\rho$-Lipschitz continuous, namely for all $\alpha,\beta\in\mathbb{R}$ we have $|\phi_i(\alpha)-\phi_i(\beta)|\leq\rho|\alpha-\beta|.$ For $\bm{a}\in\mathbb{R}^n$, let $\bm{\phi}(\bm{a})$ denote the vector $\big(\phi_1(a_1),\dots,\phi_n(a_n)\big)$. Let $\bm{\phi}\circ A=\big\{\bm{\phi}(\bm{a}):\bm{a}\in A\big\}$. Then,
	$$\mathcal{R}_n(\bm{\bm{\phi}}\circ A)\leq\rho \mathcal{R}_n(A).$$
\end{lemma}

\begin{lemma}[\bf{Lemma 26.6 of \cite{shalev}}]
	\label{lem26.6}
	For any $A\subset\mathbb{R}^n$, scale $c\in\mathbb{R}$, and vector $\bm{a}_0\in\mathbb{R}^n$, we have
	$$\mathcal{R}_n\big(\big\{c\bm{a}+\bm{a}_0\mid \bm{a}\in A\big\}\big)\leq|c|\mathcal{R}_n(A).$$
\end{lemma}

\begin{lemma}
	For any $A, B\subset\mathbb{R}^n$,  we have
	$$\mathcal{R}_n(A+B)\leq \mathcal{R}_n(A) + \mathcal{R}_n(B).$$
	Here $A+B=\big\{a+b\mid  a\in A, b\in B\big\}$.
\end{lemma}

\begin{lemma}[\bf{Lemma B.3 of \cite{barron}}]
	\label{Rn_o}
	Let $\mathcal{F}_Q=\big\{f(\bm{x};\bm{\theta})\mid \|\bm{\theta}\|_{\mathcal{P}}\leq Q\big\}$ be the set of two-layer networks with path norm bounded by $Q$, then we have
	$$\mathcal{R}_n(\mathcal{F}_Q\circ S)\leq2Q\sqrt{\frac{2\ln(2d)}{n}}.$$
\end{lemma}

\begin{lemma}[\bf{Vector valued Rademacher complexity}	\cite{maurer2016vector}]
	\label{vector_radn}
	Let $\mathcal{X}$ be any set, $(\bm{x}_1, \dots, \bm{x}_n)\in \mathcal{X}^n$, let $\mathcal{F}$ be a class of functions $\bm{f}: \mathcal{X}\to \ell_2$ and let $h_i: \ell_2\to \mathbb{R}$ have Lipschitz norm $\mbox{Lip}$, where $\ell_2$ denote the Hilbert space of square summable sequences of real numbers. Then
	\begin{equation}
		\mathbb{E}\bigg[\sup\limits_{f\in \mathcal{F}}\sum\limits_{i=1}^{n}\xi_ih_i(\bm{f}(\bm{x}_i))\bigg]\leq \sqrt{2}\mbox{Lip}\cdot\mathbb{E}\bigg[ \sup\limits_{f\in \mathcal{F}}\sum\limits_{i=1}^{n}\sum_{k=1}^{d}\xi_{ik}f_k(\bm{x}_i)\bigg],
	\end{equation}
	where $\xi_{ik}$ is an independent doubly indexed Rademacher sequence according to probability distribution $\mathbb{P}[\xi_{ik}=-1]=\mathbb{P}[\xi_{ik}=1]=\frac{1}{2}$ and $f_k(\bm{x}_i)$ is the $k$-th component of $\bm{f}(\bm{x}_i)$.
\end{lemma}
\begin{lemma}
	\label{Rad_dhq}
	Let $S=\big\{\bm{x}_1,\bm{x}_2,\cdots,\bm{x}_n\big\}$. $\big\{\bm{y}'_i\big\}_{i=1}^n$ denotes the gradient information, where $\bm{y}'_i\in \mathbb{R}^d\; \forall i\in [n]$. $$\mathcal{F}'_{Q,j}=\big\{\nabla_j f(\bm{x};\bm{\theta})\mid\Vert\bm{\theta}\Vert_{\mathcal{P}}\leq Q\big\}, ~j=1, \dots, d, \quad \mathcal{F}'_Q=\Pi_{j=1}^d \mathcal{F}'_{Q,j}.$$
	Define $\bm{g}: \mathcal{X}\to \ell_2$,  $\bm{g}(\bm{x};\bm{\theta})=\big(\partial_1f(\bm{x};\bm{\theta}\big), \partial_2f(\bm{x};\bm{\theta}), \dots, \partial_df(\bm{x};\bm{\theta}))^{\rm{T}}$ and $\tilde{\ell}_j: \ell_2 \to \mathbb{R}$, $\tilde{\ell}_j(\bm{z})=\|\bm{z}-\bm{y}'_j\|_2$ for each $j \in \big\{1,2,\cdots,n\big\}$, where $\ell_2$ denote the Hilbert space of square summable sequences of real numbers. Note that $\tilde{\ell}_j$ is $1$-Lipschitz function, then we have
	$$ 	\mathbb{E}\bigg[\sup\limits_{\|\bm{\theta}\|_{\mathcal{P}}\leq Q}\sum\limits_{i=1}^n\xi_i\tilde{\ell}_i(\bm{g}(\bm{x}_i;\bm{\theta}))\bigg]\leq \sqrt{2}\cdot \mathbb{E}\bigg[\sup\limits_{\|\bm{\theta}\|_{\mathcal{P}}\leq Q}\sum\limits_{i=1}^n\sum\limits_{k=1}^d\xi_{ik}g_k(\bm{x}_i;\bm{\theta})\bigg],$$
	where $g_k(\bm{x}_i;\bm{\theta})$ is the k-th component of $\bm{g}(\bm{x}_i;\bm{\theta})$. Moreover, denote $\tilde{\bm{\ell}} = (\tilde{\ell}_1, \tilde{\ell}_2,\cdots, \tilde{\ell}_n)$, we have
	$$ \mathcal{R}_n(\tilde{\bm{\ell}}\circ \mathcal{F}'_Q\circ S)\leq 2\sqrt{2}Qd\sqrt{\frac{2\ln(2d)}{n}}.$$
\end{lemma}

\begin{proof}
	Without loss of generality, let $\|\bm{w}_j\|_1=1,\;\forall j=\{1,\cdots,m\}.$ Apply the lemma \ref{vector_radn}, we immediately obtain
	\begin{align*}
		\mathcal{R}_n(\tilde{\bm{\ell}}\circ \mathcal{F}'_Q \circ S)&=\frac{1}{n}\mathbb{E}_{\bm{\xi}}\Big[\sup\limits_{\|\bm{\theta}\|_{\mathcal{P}}\leq Q}\sum_{i=1}^{n}\xi_i\big\|\bm{g}(\bm{x}_i;\bm{\theta})-\bm{y}'_i\big\|_2\Big]\\
		&\leq \frac{\sqrt{2}}{n}\mathbb{E}_{\bm{\xi}}\bigg[\sup\limits_{\|\bm{\theta}\|_{\mathcal{P}}\leq Q}\sum_{k=1}^d\sum\limits_{i=1}^n\xi_{ik}g_k(\bm{x}_i;\bm{\theta})\bigg]\\
		&=\frac{\sqrt{2}}{n}\mathbb{E}_{\bm{\xi}}\bigg[\sup\limits_{\|\bm{\theta}\|_{\mathcal{P}}\leq Q}\sum\limits_{k=1}^d\sum\limits_{i=1}^n\xi_{ik}\sum\limits_{j=1}^ma_j\sigma'(\bm{w}_j^T\bm{x}_i)\bm{w}_{jk}\bigg]\\
		&=\frac{\sqrt{2}}{n}\mathbb{E}_{\bm{\xi}}\bigg[\sup\limits_{\|\bm{\theta}\|_{\mathcal{P}}\leq Q}\sum\limits_{j=1}^ma_j\|\bm{w}_j\|_1\sum\limits_{k=1}^d\sum\limits_{i=1}^n\xi_{ik}\sigma'(\bm{w}_j^T\bm{x}_i)\bm{w}_{jk}\bigg]\\
		&\leq \frac{\sqrt{2}}{n}\mathbb{E}_{\bm{\xi}}\bigg[\sup\limits_{\|\bm{\theta}\|_{\mathcal{P}}\leq Q}\sum\limits_{j=1}^m|a_j|\|\bm{w}_j\|_1 \sup\limits_{\|\bm{v}\|_1=1}\bigg|\sum\limits_{k=1}^d\sum\limits_{i=1}^n\xi_{ik}\sigma'(\bm{v}^T\bm{x}_i)\bm{v}_{k}\bigg|\bigg]\\
		&\leq \frac{\sqrt{2}Q}{n}\mathbb{E}_{\bm{\xi}}\bigg[ \sup\limits_{\|\bm{v}\|_1=1}\bigg|\sum\limits_{k=1}^d\sum\limits_{i=1}^n\xi_{ik}\sigma'(\bm{v}^T\bm{x}_i)\bm{v}_{k}\bigg|\bigg]\\
		&= \frac{\sqrt{2}Q}{n}\mathbb{E}_{\bm{\eta}}\bigg[ \sup\limits_{\|\bm{v}\|_1=1}\bigg|\Big<\sum\limits_{i=1}^n\bm{\eta}_{i}\sigma'(\bm{v}^T\bm{x}_i), \bm{v}\Big>\bigg|\bigg] \quad \big(\bm{\eta}_i=(\xi_{i1},\xi_{i2},\dots, \xi_{id})\big)\\
		&\leq \frac{\sqrt{2}Q}{n}\mathbb{E}_{\bm{\eta}}\bigg[ \sup\limits_{\|\bm{v}\|_1=1}\bigg\|\sum\limits_{i=1}^n\bm{\eta}_{i}\sigma'(\bm{v}^T\bm{x}_i)\bigg\|_1\bigg]\\
		&\leq\frac{\sqrt{2}Qd}{n}\mathbb{E}_{\bm{\xi}}\bigg[ \sup\limits_{\|\bm{v}\|_1\leq1}\sum\limits_{i=1}^n\xi_{i}\sigma'(\bm{v}^T\bm{x}_i)\bigg],
	\end{align*}
	where $\xi\in[-1,1], \eta\in[-1,1]^d$. Due to the symmetry, we have
	\begin{align*}
		\mathbb{E}_{\bm{\xi}}\bigg[ \sup\limits_{\|\bm{u}\|_1\leq1}\bigg|\sum\limits_{i=1}^n\xi_{i}\sigma'(\bm{u}^T\bm{x}_i)\bigg|\bigg]&\leq \mathbb{E}_{\bm{\xi}}\bigg[ \sup\limits_{\|\bm{u}\|_1\leq1}\sum\limits_{i=1}^n\xi_{i}\sigma'(\bm{u}^T\bm{x}_i)+\sup\limits_{\|\bm{u}\|_1\leq1}\sum\limits_{i=1}^n-\xi_i\sigma'(\bm{u}^T\bm{x}_i)\bigg]\\
		&\leq2\mathbb{E}_{\bm{\xi}}\bigg[\sup\limits_{\|\bm{u}\|_1\leq1}\sum\limits_{i=1}^n\xi_i\sigma'(\bm{u}^T\bm{x}_i)\bigg].
	\end{align*}
	
	 Then apply the lemma \ref{lipschitz} and lemma \ref{vector_rn}, we obtain
	$$\mathcal{R}_n(\tilde{\bm{\ell}}\circ \mathcal{F}'_Q\circ S)\leq2\sqrt{2}Qd\sqrt{\frac{2\ln(2d)}{n}}.$$
\end{proof}

\begin{thm}[\cite{shalev}]
	Assume that for all sample $\bm{x}$  and $h$ in hypothesis space $\mathcal{H}$, the loss function $\ell:\mathcal{H}\times \mathcal{Y}\to\mathbb{R}$ satisfies $|\ell(h(\bm{x}),y)|\leq B$. Then, with probability of at least $1-\delta$, for all $h\in\mathcal{H}$, and $S=\{\bm{x}_1,\bm{x}_2,\cdots,\bm{x}_n\}$,
	\begin{equation}
		\left|\frac{1}{n}\sum\limits_{i=1}^n\ell\big(h(\bm{x}_i),y_i\big)-\mathbb{E}_{\bm{x},y}\big[\ell(h(\bm{x}),y)\big]\right|\leq 2\mathbb{E}_S\big[\mathcal{R}_n(\ell\circ\mathcal{H}\circ S)\big]+B\sqrt{\frac{2\ln(2/\delta)}{n}}.
	\end{equation}
	\label{basic thm}
\end{thm}

We are now ready to present the main results of this section.

\begin{thm}
	\label{Q_thm}
	If the assumption \ref{assumption} holds, then with probability at least $1-\delta$ we have,
		\begin{equation}
		\sup\limits_{\|\bm{\theta}\|_{\mathcal{P}}\leq Q}\Big|L(\bm{\theta})+{\beta}L'(\bm{\theta})-\big(L_n(\bm{\theta})+{\beta}L'_n(\bm{\theta})\big)\Big|\leq 4\big(1+\sqrt{2}\beta d\big)Q\sqrt{\frac{2\ln(2d)}{n}}+\left(\frac{1}{2}+\beta (Q+D)\right)\sqrt{\frac{2\ln(2/\delta)}{n}}.
	\end{equation}
\end{thm}
\begin{proof}
	Using the triangular inequality, we have
	\begin{equation}
		\label{eq_err_tol}
		\sup\limits_{\|\bm{\theta}\|_{\mathcal{P}}\leq Q}\Big|{L}(\bm{\theta})+{\beta}{L'}(\bm{\theta})-({L}_n(\bm{\theta})+{\beta}L'_n(\bm{\theta}))\Big|\leq \sup\limits_{\|\bm{\theta}\|_{\mathcal{P}}\leq Q}\big|L(\bm{\theta})-{L}_n(\bm{\theta})\big|+\sup\limits_{\|\bm{\theta}\|_{\mathcal{P}}\leq Q}\beta\big|L'(\bm{\theta})-L'_n(\bm{\theta})\big|.
	\end{equation}
	
	Define $\mathcal{H}_Q=\Big\{\ell(f(\bm{x};\bm{\theta}), y)\,\big|\,  f(\bm{x};\bm{\theta})\in \mathcal{F}_Q\Big\}$ and $\mathcal{H}'_Q=\Big\{\tilde{{\ell}}(\nabla f(\bm{x};\bm{\theta}), \bm{y}')\mid  f(\bm{x};\bm{\theta})\in \mathcal{F}_Q\Big\}$, where $\ell(f(\bm{x};\bm{\theta}), y)=\frac{1}{2}(f(\bm{x};\bm{\theta})-y)^2$ and $\tilde{{\ell}}(\nabla f(\bm{x};\bm{\theta}),\bm{y}')=\big\|\nabla f(\bm{x};\bm{\theta})-\bm{y}'\big\|_2$.
	
	Note that $f(\bm{x};\bm{\theta})\in[0,1]$,  $\ell(\cdot, y)$ is 1-Lipschitz continuous,
	\begin{equation*}
		\begin{aligned}
			&\frac{1}{2}\left|(f(\bm{x}_i;\bm{\theta}_f)-y_i)^2-(g(\bm{x}_i;\bm{\theta}_g)-y_i)^2\right|\\
			=&\frac{1}{2} \left|f(\bm{x}_i;\bm{\theta}_f)-g(\bm{x}_i;\bm{\theta}_g)\right|\cdot\left|f(\bm{x}_i;\bm{\theta}_f)+g(\bm{x}_i;\bm{\theta}_g)-2y_i\right|\leq|f(\bm{x}_i;\bm{\theta}_f)-g(\bm{x}_i;\bm{\theta}_g)|.
		\end{aligned}
	\end{equation*}
	Follow from Lemma \ref{Rn_o} and Lemma \ref{lipschitz}, we have
	\begin{equation}
		\mathcal{R}_n(\mathcal{H}_Q)=\frac{1}{n}\mathbb{E}_{\bm{\xi}}\bigg[\sup\limits_{f\in\mathcal{F}_Q}\sum\limits_{i=1}^n\xi_i\frac{1}{2}\left|f(\bm{x}_i;\bm{\theta})-y_i\right|^2\bigg]\leq 2Q\sqrt{\frac{2\ln(2d)}{n}}.
	\end{equation}
	And $\ell$ is bounded, $\ell(f(\bm{x};\bm{\theta}), y)=\frac{1}{2}(f(\bm{x};\bm{\theta})-y)^2\leq\frac{1}{2}$. Hence applying the Theorem \ref{basic thm}, we can obtain,
	\begin{equation}
		\sup\limits_{\|\bm{\theta}\|_{\mathcal{P}}\leq Q}\big|L(\bm{\theta})-L_n(\bm{\theta})\big|\leq 4Q\sqrt{\frac{2\ln(2d)}{n}}+\frac{1}{2}\sqrt{\frac{2\ln(2/\delta)}{n}}.
		\label{eq_err_L}
	\end{equation}
	Similarly,  $\tilde{\ell}(\nabla f(\bm{x};\bm{\theta}),\bm{y}')=\big\|\nabla f(\bm{x};\bm{\theta})-\bm{y}'\big\|_2$ is 1-Lipschitz with respect to $\nabla f(\bm{x};\bm{\theta})$,
	$$\bigg|\big\|\nabla f(\bm{x}_i;\bm{\theta}_f)-\bm{y}'_i\big\|_2-\big\|\nabla g(\bm{x}_i;\bm{\theta}_g)-\bm{y}'_i\big\|_2\bigg|\leq\big\|\nabla f(\bm{x}_i;\bm{\theta}_f)-\nabla g(\bm{x}_i;\bm{\theta}_g)\big\|_2.$$
	Let $\tilde{\bm{\ell}}=(\tilde{\ell}_1, \dots, \tilde{\ell}_n)$ in which $\tilde{\ell}_i(\cdot)=\big\|\cdot-\bm{y}'_i\big\|_2$ for $i=1,\dots, n$. We can obtain the estimation of $\mathcal{R}_n(\mathcal{H}'_Q)$ directly from Lemma \ref{Rad_dhq},
	\begin{equation}
		\mathcal{R}_n(\mathcal{H}'_Q)=\frac{1}{n}\mathbb{E}_{\bm{\xi}}\bigg[\sup\limits_{f\in\mathcal{F}_Q}\sum_{i=1}^{n}\xi_i\Big\|\nabla f(\bm{x}_i;\bm{\theta})-\bm{y}'_i\Big\|_2\bigg]\leq 2\sqrt{2}Qd\sqrt{\frac{2\ln(2d)}{n}}.
	\end{equation}
	And $\tilde{{\ell}}(\nabla f(\bm{x};\bm{\theta}),\bm{y}')$ is bounded since we assume the gradient of objective function is bounded by a positive constant $D$:
	\begin{align}
		f(\bm{x};\bm{\theta}) & = \sum\limits_{k=1}^ma_k\sigma(\bm{w}_k^T\bm{x}), \nonumber\\
		\big\|\nabla f(\bm{x};\bm{\theta})\big\|^2_2 & = \Big\|\sum\limits_{k=1}^m a_k\sigma'\bm{w}_k^T\Big\|_2^2\leq\Big\|\sum\limits_{k=1}^ma_k\sigma'\bm{w}_k^T\Big\|_1^2\leq\Big\|\sum\limits_{k=1}^ma_k\bm{w}_k^T\Big\|_1^2\leq \big\|\bm{\theta}_f\big\|^2_{\mathcal{P}},\\
		\tilde{\ell}(\nabla f(\bm{x};\bm{\theta}),\bm{y}')&= \big\|\nabla f(\bm{x};\bm{\theta})-\bm{y}'\big\|_2\leq \big\|\nabla f(\bm{x};\bm{\theta})\big\|_2+\big\|\bm{y}'\big\| \leq  \big\|\bm{\theta}_f\big\|_{\mathcal{P}}+D \leq Q+D.
		\label{gradient_bound}
	\end{align}
	Thus,
	\begin{equation}
		\sup\limits_{\|\bm{\theta}\|_{\mathcal{P}}\leq Q}\big|L'(\bm{\theta})-L'_n(\bm{\theta})\big|\leq 4\sqrt{2}Qd\sqrt{\frac{2\ln(2d)}{n}}+(Q+D)\sqrt{\frac{2\ln(2/\delta)}{n}}.
		\label{eq_err_dL}
	\end{equation}
	
The desired result follows by plugging equations \ref{eq_err_L} and \ref{eq_err_dL} into equation \ref{eq_err_tol}, and the proof is completed.
\end{proof}

We next present a posterior generalization bound by relaxing such restrictions.

\begin{thm}[\bf{A posterior generalization bound}]
	\label{post_thm}
Assume that Assumption \ref{assumption} holds, then for any $\delta>0$, with probability at least $1-\delta$ over the choice of the training set $S$, we have, for any two-layer network $f_m(\cdot,\bm{\theta})$,
	\begin{equation}
		\begin{split}
			\bigg|L(\bm{\theta})+{\beta}L'(\bm{\theta})-\big(L_n(\bm{\theta})+{\beta}L'_n(\bm{\theta})\big)\bigg|\leq& 4\Big(1+\sqrt{2}\beta d\Big)\Big(\|\bm{\theta}\|_{\mathcal{P}}+1\Big)\sqrt{\frac{2\ln(2d)}{n}}\\
			&+\bigg(\frac{1}{2}+\beta \Big(\|\bm{\theta}\|_{\mathcal{P}}+1+D\Big)\bigg)\sqrt{\frac{2\ln(2c(\|\bm{\theta}\|_{\mathcal{P}}+1)^2/\delta)}{n}},
		\end{split}
	\end{equation}
	where $c=\sum\limits_{k=1}^\infty1/k^2$.
\end{thm}
\begin{proof}
	Consider the decomposition $\mathcal{F}=\cup_{k=1}^\infty\mathcal{F}_k$, where $\mathcal{F}_k=\big\{f(\bm{x};\bm{\theta})\mid \|\bm{\theta}\|_{\mathcal{P}}\leq k\big\}$. Let $\delta_k=\delta / (ck^2)$ where $c=\sum_{k=1}^\infty\frac{1}{k^2}$. According to the theorem \ref{Q_thm}, if we fix $k$ in advance, then with probability at least $1-\delta_k$ over the choice of $S$, we have
	\begin{align}
	\Big|L(\bm{\theta})+{\beta}L'(\bm{\theta})-\big(L_n(\bm{\theta})+{\beta}L'_n(\bm{\theta})\big)\Big| \leq 4\big(1+\sqrt{2}\beta d\big)k\sqrt{\frac{2\ln(2d)}{n}}+\bigg(\frac{1}{2}+\beta \big(k+D\big)\bigg)\sqrt{\frac{2\ln(2/\delta_k)}{n}}.
		\label{lthm}
	\end{align}
	Therefore $\mathbb{P}\big(\big\{\text{inequality} ~(\ref{lthm}) ~\text{unholds} \big\}\big)\leq\sum_{k=1}^\infty\delta_k\eqqcolon\delta$, namely, $\mathbb{P}\big(\big\{\mbox{inequality } (\ref{lthm}) \mbox{ holds for all } k \big\}\big)\geq1-\delta$. In other words, with probability at least $1-\delta$, the inequality (\ref{lthm}) holds for all $k$.
	
	Given an arbitrary set of parameters $\bm{\theta}$, denote $k_0=\min\big\{k\mid \|\bm{\theta}\|_{\mathcal{P}}\leq k\big\}$, then $k_0\leq \|\bm{\theta}\|_{\mathcal{P}}+1$. Inequality (\ref{lthm}) implies that
	\begin{align*}	\Big|L(\bm{\theta})+{\beta}L'(\bm{\theta})-\big(L_n(\bm{\theta})+{\beta}L'_n(\bm{\theta})\big)\Big| \leq & 4\big(1+\sqrt{2}\beta d\big)k_0\sqrt{\frac{2\ln(2d)}{n}}+\bigg(\frac{1}{2}+\beta \Big(k_0+D\Big)\bigg)\sqrt{\frac{2\ln(2ck_0^2/\delta)}{n}}.\\
		\leq & 4\big(1+\sqrt{2}\beta d\big)\big(\|\bm{\theta}\|_{\mathcal{P}}+1\big)\sqrt{\frac{2\ln(2d)}{n}}\\
		&+\left(\frac{1}{2}+\beta \big(\|\bm{\theta}\|_{\mathcal{P}}+1+D\big)\right)\sqrt{\frac{2\ln(2c(\|\bm{\theta}\|_{\mathcal{P}}+1)^2/\delta)}{n}}.
	\end{align*}
This completes the proof.	
\end{proof}
We notice that the generalization gap is bounded roughly by $d\|\bm{\theta}\|_{\mathcal{P}}\sqrt{\frac{\ln(2d)}{n}}+\|\bm{\theta}\|_{\mathcal{P}}\frac{\ln(\|\bm{\theta}\|_{\mathcal{P}})}{\sqrt{n}}$.
\subsection{A upper bound for the empirical risk}
Recall the approximation property, there exists a two-layer neural network $f(\cdot;\tilde{\bm{\theta}})$ whose path norm is independent of the network width, while achieving the optimal approximation error. Furthermore, this path norm can also be used to bound the generalization gap (Theorem \ref{post_thm}).
In this section, we want to estimate the gradient regularized risk of $\tilde{\bm{\theta}}$. To this end, we first assume the norm of gradient of target function can be bounded by a constant $D$.

Let $\hat{\gamma}_p(f)=\max\{1,\gamma_p(f)\}$, where $d$ is the dimension of input data.

\begin{thm}
	Let $\tilde{\bm{\theta}}$ be the network mentioned in Theorem \ref{appro_thm}, then with probability at least $1-\delta$, we have
	\begin{equation}
		J_{n,\beta}(\tilde{\bm{\theta}})\lesssim \frac{\gamma_2^2(f^*)}{m}+\beta\frac{1}{\sqrt{m}}\gamma_2(f^*)+\beta \hat{\gamma}_2(f^*)\big(\hat{\gamma}_2(f^*)+d\big)\sqrt{\frac{2\ln(2d)}{n}}+\beta\hat{\gamma}_2(f^*) \sqrt{\frac{\ln(2c/\delta)}{n}},
		\label{reg_risk_eq}
	\end{equation}
	\label{reg_risk_thm}
	where \begin{equation*}
		J_{n,\beta}(\tilde{\bm{\theta}}) = L_n(\tilde{\bm{\theta}})+\beta L'_n(\tilde{\bm{\theta}})\quad \mbox{and}\quad c = \sum_{k=1}^\infty \frac{1}{k^2}.
	\end{equation*}
\end{thm}
\begin{proof}

According to the definition of regularized model, the properties that $\big\|\tilde{\bm{\theta}}\big\|_{\mathcal{P}}\leq2\gamma_2(f^*), L(\tilde{\bm{\theta}})\leq \frac{3\gamma^2_2(f^*)}{m}$, $\big( L' (\tilde{\bm{\theta}})\big)^2\leq\mathbb{E}_x\Big[\big\|\nabla f(\bm{x})-\nabla f(\bm{x};\tilde{\bm{\theta}})\big\|^2_2\Big]\leq\frac{7\gamma_2^2(f^*)}{m}$ and the posteriori error bound, the regularized risk of $\tilde{\bm{\theta}}$ satisfy
\begin{align}
	J_{n,\beta}(\tilde{\bm{\theta}}) = & L _n(\tilde{\bm{\theta}})+\beta L' _n(\tilde{\bm{\theta}})\nonumber\\
	\leq & L (\tilde{\bm{\theta}})+\beta L' (\tilde{\bm{\theta}})+4\big(1+\sqrt{2}\beta d\big)\big(\big\|\tilde{\bm{\theta}}\big\|_{\mathcal{P}}+1\big)\sqrt{\frac{2\ln(2d)}{n}} \\
	&+\bigg(\frac{1}{2}+\beta\big(\big\|\tilde{\bm{\theta}}\big\|_{\mathcal{P}}+1+D\big)\bigg)\sqrt{\frac{2\ln(2c(\big\|\tilde{\bm{\theta}}\big\|_{\mathcal{P}}+1)^2/\delta)}{n}}\nonumber\\
	\leq &  L (\tilde{\bm{\theta}})+\beta L' (\tilde{\bm{\theta}})+4\big(1+\sqrt{2}\beta d\big)\big(2\gamma_2(f^*)+1\big)\sqrt{\frac{2\ln(2d)}{n}} \\ 
	&+\bigg(\frac{1}{2}+\beta\big(2\gamma_2(f^*)+1+D\big)\bigg)\sqrt{\frac{2\ln(2c(1+2\gamma_2(f^*))^2/\delta)}{n}}.
	\label{reg_es}
\end{align}
The last term can be simplified by using $\sqrt{a+b}\leq\sqrt{a}+\sqrt{b}$ and $\ln(a)\leq a$ for $a\geq 0, b\geq0$. Thus we have
\begin{align*}
	\sqrt{2\ln(2c(1+2\gamma_2(f^*))^2/\delta)} & = \sqrt{2\ln(2c/\delta)+2\ln(1+2\gamma_2(f^*))^2}\\
	& \leq \sqrt{2\ln(2c/\delta)}+\sqrt{2\ln(1+2\gamma_2(f^*))^2}\\
	& \leq \sqrt{2\ln(2c/\delta)}+\sqrt{2\ln(3\hat{\gamma}_2(f^*))^2}\\
	& \leq \sqrt{2\ln(2c/\delta)}+3\sqrt{2}\hat{\gamma}_2(f^*).
\end{align*}
Plugging it into Equation (\ref{reg_es}), we obtain
\begin{align*}
	J_{n,\beta}(\tilde{\bm{\theta}}) &\leq  L (\tilde{\bm{\theta}})+\beta L' (\tilde{\bm{\theta}})+4\big(1+\sqrt{2}\beta d\big)\big(2\gamma_2(f^*)+1\big)\sqrt{\frac{2\ln(2d)}{n}}\\
& \quad + \left(\frac{1}{2} +\beta\big(2\gamma_2(f^*)+1+D\big)\right)\left(\sqrt{\frac{2\ln(2c/\delta)}{n}}+3\sqrt{\frac{2}{n}}\hat{\gamma}_2(f^*)\right).
\end{align*}
Thus after some simplifications, we have
\begin{align*}
	J_{n,\beta}(\tilde{\bm{\theta}})\lesssim \frac{\gamma_2^2(f^*)}{m}+\beta\frac{1}{\sqrt{m}}\gamma_2(f^*)+\beta \hat{\gamma}_2(f^*)\big(\hat{\gamma}_2(f^*)+d\big)\sqrt{\frac{2\ln(2d)}{n}}+\beta\hat{\gamma}_2(f^*) \sqrt{\frac{\ln(2c/\delta)}{n}}.
\end{align*}
This completes the proof.
\end{proof}
According to the definition of $\bm{\theta}_{n,\beta}$, we have
	$J_{n,\beta}(\bm{\theta}_{n,\beta})\leq J_{n,\beta}(\tilde{\bm{\theta}})$. Thus the above theorem gives a upper bound for $J_{n,\beta}(\bm{\theta}_{n,\beta})$.

\section{Applications to uncertainty quantification}
We now consider the application of the Gradient-enhanced DNNs approach for uncertainty quantification.

\subsection{Gradient-enhanced uncertainty quantification}
In complex engineering systems, mathematical models can only serve as simplified and reduced representations of true physics, and the effect of some uncertainties, such as boundary/initial conditions and parameter values, can be significant. Uncertainty Quantification (UQ) aims to develop numerical tools that can accurately predict quantities of interest (QoI) and facilitate the quantitative validation of the simulation model. Generally, we use differential equations to model complex systems on a domain $\Omega$, in which the uncertainty sources are represented by $\Xi$. The solution $u$ is governed by the PDEs
\begin{equation}
\begin{aligned}
&\mathcal{L}\big(\bm{x}, \Xi; u(\bm{x}, \Xi)\big)=0,\quad & \bm{x}\in \Omega,\\
&\mathcal{B}\big(\bm{x}, \Xi; u(\bm{x}, \Xi)\big)=0,\quad & \bm{x}\in \partial \Omega,		
\end{aligned}
\end{equation}
where $\mathcal{L},\mathcal{B}$ are differential and boundary operators. Our goal is to approximate the QoI $u(\bm{x}_0, \Xi)$ for some fixed spatial location $\bm{x}_0$. To reduce the notation, we simply write $u(\Xi)$.  In many applications, the dimension of random variable $\Xi$ is very high and can be characterized by  a $d$-dimensional random variable. Hence DNNs are good candidates for such problems.

We consider inclusion of gradient measurements with respect to random variables $\Xi$, i.e., $\partial u/ \partial \Xi_k, k=1,2,\cdots,d.$ The gradient measurements can usually be obtained in a relatively inexpensive way via the adjoint method \cite{luchini2014adjoint}.
\subsection{Numerical examples}
We next provide several numerical experiments to show the performance of the proposed approach. We compare the performance between original neural networks without gradient input and gradient-enhanced neural networks.  For simplicity, we say that method is X\% gradient-enhanced if X\% of samples contain derivative information with respect to all dimensions. In our test, each neural network contains two layers with 1000 hidden neurons useless specific. And the hyper parameter $\beta$ which is used to balance two part losses introduced by function values and derivative information respectively is set to 10.
We initialize all trainable parameters using the Glorot normal scheme. For the training procedure, we use the Adam optimizer. To quantitatively evaluate the accuracy of the numerical solution, we shall consider the relative $L^2$ error $\Vert u_\theta -u\Vert_2/ \Vert u\Vert_2 $ , where $u$ and $u_\theta$ denote the ground truth and predicted solution. All numerical tests are implemented in Pytorch.
\subsubsection{Function approximations}
Before applying the gradient-enhanced method to uncertainty quantification, we first demonstrate the effectiveness of our approach in approximating high dimensional functions. More precisely,  we consider the Gaussian function:
 $$f_1(\bm{x})=\exp \bigg(-\sum_{i=1}^d x_i^2\bigg), \quad \bm{x}=(x_1,x_2,\cdots, x_d) \in[-1,1]^d,$$
and the polynomial function:
 $$f_2(\bm{x})=\sum_{i=1}^{d/2}x_i x_{i+1}, \quad \bm{x}=(x_1,x_2,\cdots, x_d)\in[-1, 1]^d.$$
 \begin{figure}
 	\centering
 	\includegraphics[scale=0.5]{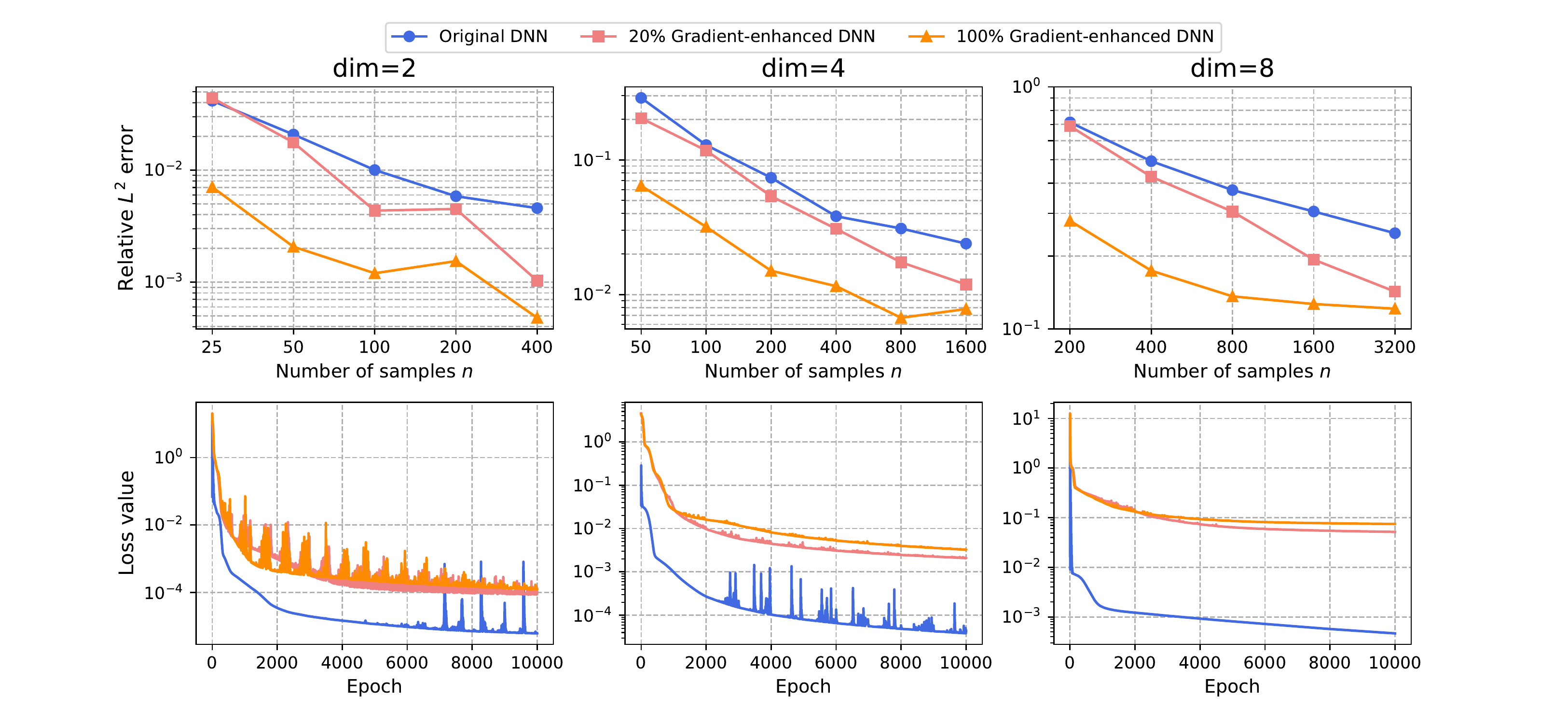}
 	\caption{Approximation of $f_1(\bm{x})$. Top: The relative $L^2$ errors against number of samples. Bottom: The loss functions against increasing epochs with the number of samples 400 for $d=2$, 1600 for $d=4$ and $3200$ for $d=8$.}
 	\label{fig:functionapproximation1}
 \end{figure}
 \begin{figure}
 	\centering
 	\includegraphics[scale=0.5]{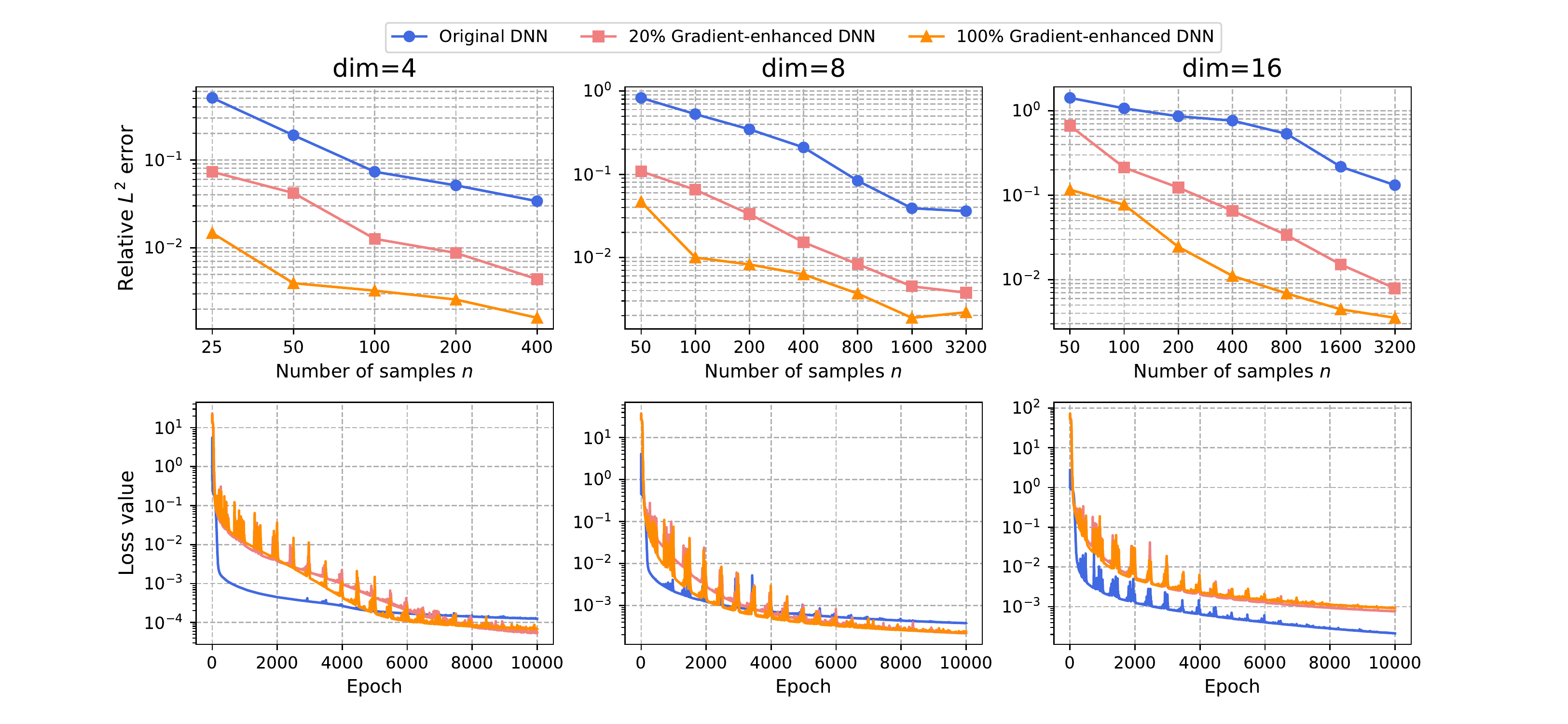}
 	\caption{Approximation of $f_2(\bm{x})$. Top: The relative $L^2$ errors against number of samples. Bottom: The loss functions against increasing epochs with the number of samples 400 for dim=4 and 3200 for dim=8, 16.}
 	\label{fig:functionapproximation2}
 \end{figure}

 For these two test functions, we assume that samples 
 $\{\bm{x}_i\}_{i=1}^n$ are uniformly distributed in $[-1,1]^d$, $y_i$ is the observation of target function at $\bm{x}_i$ and $\bm{y}_i'$ is the corresponding derivative. Thus $\{\bm{x}_i,y_i\}_{i=1}^n$ compose the training data for original DNN. $\{\bm{x}_i, y_i,\bm{y}'_i\}_{i=1}^n$ compose the training data for 100\% gradient-enhanced DNN. And $\{\bm{x}_i,y_i\}_{i=1}^n\cup\{\hat{\bm{x}}_j, \hat{\bm{y}}'_j\}^m_{j=1}$ compose the training data for 20\% gradient-enhanced DNN where $m$ is the rounding off of $20\%n$ and $\{\hat{\bm{x}}_j\}^m_{j=1}$ is randomly chosen from $\{\bm{x}_i\}_{i=1}^n$. The learning rate for Adam optimizer is set to 0.01 with 20\% decay each 500 steps.
 
 For Gaussian function $f_1(\bm{x})$, we consider the cases that $d=2,4,8$. 
 The relative $L^2$ errors against the number of samples $n$ are presented in top row of Fig. \ref{fig:functionapproximation1}. The use of gradient information can indeed improve the accuracy, and furthermore, the more gradient information is included, the better accuracy obtained. What's more, we investigate the loss functions of different models for $d=2$ with 400 samples, $d=4$ with 1600 samples and $d=8$ with 3200 samples, which are depicted in the bottom row of Fig. \ref{fig:functionapproximation1}. 

 For the polynomial function $f_2(\bm{x})$, we set the dimension to 4, 8, and 16. Similar to Gaussian function, we present the relative $L^2$ errors for different dimensions in Fig. \ref{fig:functionapproximation2}, which again show that the gradient information regularized term can greatly enhance the approximation accuracy. The loss functions for $d=4$ with 400 samples and $d=8,16$ with 3200 samples are provided in the bottom row of  Fig. \ref{fig:functionapproximation2}. It can be observed that the loss function of gradient-enhanced methods may be smaller than original DNN as the iteration number increases, verifying the strength of gradient-enhanced methods.
 \subsubsection{Elliptic PDE with random inputs}
 We now consider the following stochastic elliptic PDE problem
 \begin{equation}
 	\label{stochastic_elliptic}
 	\left\{
 	\begin{aligned}
 		-\nabla\cdot (a(\bm{x},\omega)\nabla u(\bm{x},\omega))=f(\bm{x},\omega)\quad &\mbox{in } \mathcal{D}\times \Omega,\\
 		 u(\bm{x},\omega)=0\quad &\mbox{on }\partial \mathcal{D}\times \Omega,
 	\end{aligned}
 	\right.
 \end{equation}
where $\mathcal{D}=[0,1]^2, \bm{x}=(x_1,x_2)$ is a spatial coordinate, $f(\bm{x},\omega)$ is a deterministic force term $f(\bm{x},\omega)=\cos(x_1)\sin(x_2)$. The random diffusion coefficient $a(\bm{x},\omega)=a_{d}(\bm{x},\omega)$ with one-dimensional spatial dependence takes the form \cite{FN},
\begin{equation}
	\log(a_{d}(\bm{x},\omega)-0.5)=1+ Y_1(\omega)\left(\frac{\sqrt{\pi}L}{2}\right)^{1/2}+\sum_{k=2}^{d}\zeta_k \phi_k(\bm{x})Y_k(\omega),
\end{equation}
where
	\begin{equation*}
		\zeta_k \coloneqq (\sqrt{\pi}L)^{1/2}\exp\left( \frac{-\big(\lfloor \frac{k}{2}\rfloor \pi L\big)^2}{8}\right)\quad \mbox{if } k>1 \mbox{ and } L=\frac{1}{12},
	\end{equation*}
and $\phi_k(\bm{x})$ only depends on the first component of $\bm{x}$,
\begin{equation}
	\phi_k(\bm{x})\coloneqq \left\{
	\begin{array}{cc}
		\sin\left(\lfloor \frac{k}{2}\rfloor \pi x_1\right) & \mbox{if}\; k \; \mbox{even},\\
		\cos\left(\lfloor \frac{k}{2}\rfloor \pi x_1\right) & \mbox{if}\; k \; \mbox{odd}.
	\end{array}
	\right.
\end{equation}
Here $\{Y_k(\omega)\}_{k=1}^d $ are independent random variables uniformly distributed in the interval $[-1, 1]$. In the following we approximate the QoI $q$ defined by $q(\omega)=u((0.5, 0.5), \omega)$ which is the solution of equation \eqref{stochastic_elliptic} at location $\bm{x}=(0.5,0.5)$. Denote $\Psi(\omega)=((Y_1(\omega),\dots,Y_d(\omega))$. The derivatives $\mathrm{d}q/\mathrm{d}\Psi=\partial u(\bm{x},\omega)/\partial \Psi$ are computed by the adjoint sensitivity method. Both forward and adjoint solvers are implemented in the finite element method (FEM) project Fenics \cite{logg2012automated}. In numerical tests, $\{\Psi(\omega_i)\}_{i=1}^n$ are generated from a uniform distribution in $[-1,1]^d$, and we solve the forward PDE 320 times for $d=5$ and 1600 times for $d=10$. Notice that each partial derivative leads to an adjoint equation, then the number of adjoint equations needed to solve is $d$ times that of forward equations. It is worth mentioning that the cost of generating derivatives of $q(\omega)$ in elliptic PDE \eqref{stochastic_elliptic} is negligible since they share the same stiff matrix with $q$. 

After obtaining the function values as well as the corresponding gradient information, we apply the original DNN, 20\% gradient-enhanced DNN and 100\% gradient-enhanced DNN to approximate the QoI $q(\omega)$. The learning rate for Adam optimizer is 0.001 with half decay each 1000 steps. The relative $L^2$ errors for $d=5$, 10 are presented in the top row of 
Fig. \ref{fig:stochasticellipticuq}. We also provide the loss functions of different models for $d=5$ with the number of samples 320 and $d=10$ with the number of samples 1600 in the bottom row of Fig. \ref{fig:stochasticellipticuq}. All cases verify that gradient-enhanced methods significantly out-perform original approach. We can achieve the same accuracy using much fewer training samples.
 \begin{figure}[!h]
 	\centering
 	\includegraphics[scale=0.6]{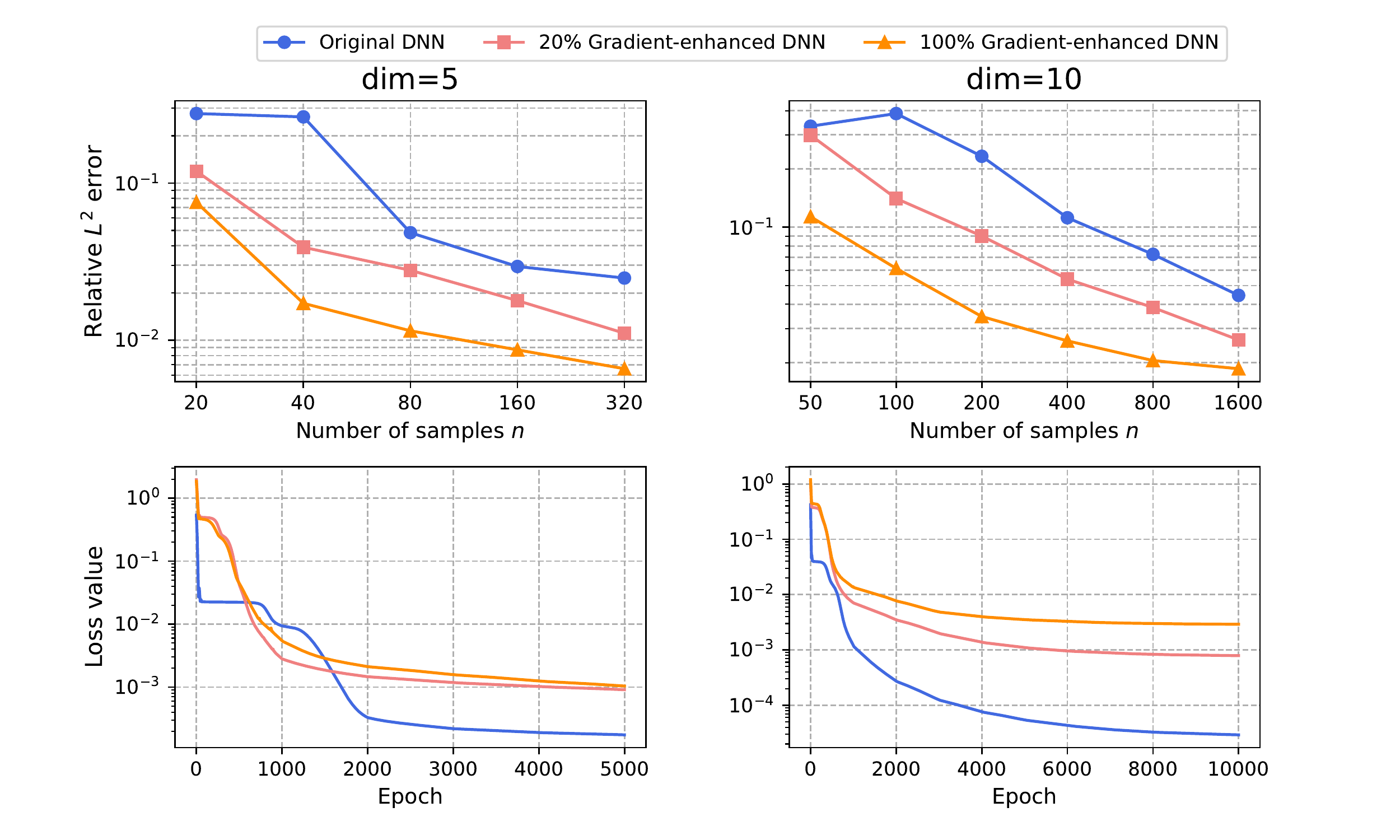}
 	\caption{Top: The relative $L^2$ errors against number of samples for $N=5, 10$. Bottom: The loss functions against increasing epochs for $d=5$ with 320 samples and $d=10$ with 1600 samples.}
 	\label{fig:stochasticellipticuq}
 \end{figure}

%
%
%

\section{Conclusion}
We have proposed gradient-enhanced deep neural networks (DNNs) approximations for function approximations and uncertainty quantification. In our approach, the gradient information is included as a regularization term. For this approach, we present similar posterior estimates (by the two-layer neural networks) as those in the path-norm regularized DNNs approximations. We also discuss the application of this approach to gradient-enhanced uncertainty quantification, and numerical experiments show that the proposed approach can outperform the traditional DNNs approach in many cases of interests. The discussion in this work is limited to supervised learning where labeled data are available, and in our future work, we shall consider to apply this gradient-enhanced idea to unsupervised learning where the physical equation is considered to yield the loss function.

\section*{Acknowledge}
We would like to thank Professor Tao Zhou of Chinese Academy of Sciences for bringing this topic to our attention and for his encouragement and helpful discussion.

\bibliographystyle{plainnat}

\end{document}